\documentclass{article}

\usepackage[english]{babel}
\usepackage[utf8]{inputenc}
\usepackage[unicode]{hyperref}

\usepackage{amsmath}
\usepackage{amssymb}
\usepackage{enumerate}
\usepackage{bbold} 
\usepackage{paralist} 
\usepackage{float} 
\usepackage[labelfont=bf]{caption} 

\usepackage{multirow} 
\usepackage{colortbl} 
\def\fainthline{\arrayrulecolor{black!10!white}\hline}

\usepackage{mythm}
\usepackage[flowcharts]{aixi}

\def\MM{{\overline{M}}}     

\def\dmid{\parallel}   

\def\AIXI{\mathrm{AIXI}}
\def\AIMU{\mathrm{AIMU}}
\def\AINU{\mathrm{AINU}}
\def\AINULT{{\AINU_\textrm{LT}}}
\def\AINUDC{{\AINU_\textrm{DC}}}
\def\AIMULT{{\AIMU_\textrm{LT}}}
\def\AIMUDC{{\AIMU_\textrm{DC}}}
\def\AIXILT{{\AIXI_\textrm{LT}}}
\def\AIXIDC{{\AIXI_\textrm{DC}}}

\def\O{\mathcal{O}} 
\def\one{\mathbb{1}} 

\AtBeginDocument{

}

\title{On the Computability of AIXI}
\author{Jan Leike and Marcus Hutter}

\begin{document}

\maketitle

\begin{abstract}%
How could we solve the machine learning and the artificial intelligence problem
if we had infinite computation?
Solomonoff induction and the reinforcement learning agent AIXI
are proposed answers to this question.
Both are known to be incomputable.
In this paper, we quantify this using the arithmetical hierarchy,
and prove upper and corresponding lower bounds for incomputability.
We show that AIXI is not limit computable,
thus it cannot be approximated using finite computation.
Our main result is a limit-computable $\varepsilon$-optimal version of AIXI
with infinite horizon that maximizes expected rewards.
\end{abstract}


\section{Introduction}
\label{sec:introduction}

Given infinite computation power,
many traditional AI problems become trivial:
playing chess, go, or backgammon can be solved
by exhaustive expansion of the game tree.
Yet other problems seem difficult still; for example,
predicting the stock market, driving a car, or babysitting your nephew.
How can we solve these problems in theory?
A proposed answer to this question is the agent
AIXI~\cite{Hutter:2000,Hutter:2005}.
As a \emph{reinforcement learning agent},
its goal is to maximize cumulative (discounted) rewards
obtained from the environment~\cite{SB:1998}.

The basis of AIXI is Solomonoff's theory of learning~\cite{Solomonoff:1964,Solomonoff:1978,LV:2008},
also called \emph{Solomonoff induction}.
It arguably solves the induction problem~\cite{RH:2011}:
for data drawn from a computable measure $\mu$,
Solomonoff induction will converge to the correct belief about any hypothesis~\cite{BD:1962,RH:2011}.
Moreover, convergence is extremely fast in the sense that
Solomonoff induction will make a total of at most $E + O(\sqrt{E})$ errors
when predicting the next data points,
where $E$ is the number of errors of
the informed predictor that knows $\mu$~\cite{Hutter:2001error}.
While learning the environment according to Solomonoff's theory,
AIXI selects actions by running an expectimax-search
for maximum cumulative discounted rewards.
It is clear that AIXI can only serve as an ideal,
yet recently it has inspired some impressive applications~\cite{VNHUS:2011}.

Both Solomonoff induction and AIXI are known to be incomputable.
But not all incomputabilities are equal.
The \emph{arithmetical hierarchy} specifies different levels of computability
based on \emph{oracle machines}:
each level in the arithmetical hierarchy
is computed by a Turing machine which
may query a halting oracle for the respective lower level.

We posit that any ideal for a `perfect agent'
needs to be \emph{limit computable} ($\Delta^0_2$).
The class of limit computable functions is the class of functions that
admit an \emph{anytime algorithm}.
It is the highest level of the arithmetical hierarchy
which can be approximated using a regular Turing machine.
If this criterion is not met, our model would be useless
to guide practical research.

For MDPs, planning is already P-complete for finite and infinite horizons
\cite{PT:1987}.
In POMDPs, planning is undecidable~\cite{MHC:1999,MHC:2003}.
The existence of a policy whose expected value exceeds a given threshold is
PSPACE-complete~\cite{MGLA:2000},
even for purely epistemic POMDPs
in which actions do not change the hidden state~\cite{SLR:2007}.
In this paper we derive hardness results
for planning in general semicomputable environments;
this environment class is even more general than POMDPs.
We show that finding an optimal policy is $\Pi^0_2$-hard and
finding an $\varepsilon$-optimal policy is undecidable.

Moreover, we show that by default, AIXI is not limit computable.
The reason is twofold:
First, when picking the next action,
two or more actions might have the same value (expected future rewards).
The choice between them is easy,
but determining whether such a tie exists is difficult.
Second,
in case of an infinite horizon (using discounting),
the iterative definition of the value function~\cite[Def.\ 5.30]{Hutter:2005}
conditions on surviving forever.
The first problem can be circumvented
by settling for an $\varepsilon$-optimal agent.
We show that the second problem can be solved
by using the recursive instead of the iterative definition of the value function.
With this we get a limit-computable agent with infinite horizon.
\autoref{tab:complexity-agents} and \autoref{tab:complexity-induction} summarize
our computability results.

\begin{table}[t]
\begin{center}
\renewcommand{\arraystretch}{1.2}
\setlength{\tabcolsep}{5pt}
\begin{tabular}{llll}
Model & $\gamma$ & Optimal & $\varepsilon$-Optimal \\
\hline
\multirow{2}{*}{Iterative AINU}
      & DC & $\Delta^0_4$, $\Sigma^0_3$-hard
           & $\Delta^0_3$, $\Pi^0_2$-hard \\
      & LT & $\Delta^0_3$, $\Pi^0_2$-hard
           & $\Delta^0_2$, $\Sigma^0_1$-hard \\
\fainthline
\multirow{2}{*}{Iterative AIXI}
      & DC & $\Delta^0_4$, $\Pi^0_2$-hard
           & $\Delta^0_3$, $\Pi^0_2$-hard \\
      & LT & $\Delta^0_3$, $\Sigma^0_1$-hard
           & $\Delta^0_2$, $\Sigma^0_1$-hard \\
\fainthline
\multirow{2}{*}{Iterative AIMU} 
      & DC & $\Delta^0_2$
           & $\Delta^0_1$ \\
      & LT & $\Delta^0_2$
           & $\Delta^0_1$ \\
\fainthline
\multirow{2}{*}{Recursive AINU}
      & DC & $\Delta^0_3$, $\Pi^0_2$-hard
           & $\Delta^0_2$, $\Sigma^0_1$-hard \\
      & LT & $\Delta^0_3$, $\Pi^0_2$-hard
           & $\Delta^0_2$, $\Sigma^0_1$-hard \\
\fainthline
\multirow{2}{*}{Recursive AIXI}
      & DC & $\Delta^0_3$, $\Sigma^0_1$-hard
           & $\Delta^0_2$, $\Sigma^0_1$-hard \\
      & LT & $\Delta^0_3$, $\Sigma^0_1$-hard
           & $\Delta^0_2$, $\Sigma^0_1$-hard \\
\fainthline
\multirow{2}{*}{Recursive AIMU}
      & DC & $\Delta^0_2$
           & $\Delta^0_1$ \\
      & LT & $\Delta^0_2$
           & $\Delta^0_1$ \\
\end{tabular}
\end{center}
\caption{
Computability results for different agent models
derived in \autoref{sec:complexity-aixi}.
DC means general discounting,
a lower semicomputable discount function $\gamma$;
LT means finite lifetime,
undiscounted rewards up to a fixed lifetime $m$.
Hardness results for AIXI are with respect to
a specific universal Turing machine;
hardness results for AINU are with respect to
a specific environment $\nu \in \M$.
}
\label{tab:complexity-agents}
\end{table}

\section{Preliminaries}
\label{sec:preliminaries}

\subsection{The Arithmetical Hierarchy}

A set $A \subseteq \mathbb{N}$ is $\Sigma^0_n$ iff
there is a computable relation $S$ such that
\begin{equation}\label{eq:def-Sigma^0_n}
k \in A
\;\Longleftrightarrow\;
\exists k_1 \forall k_2 \ldots Q_n k_n\; S(k, k_1, \ldots, k_n)
\end{equation}
where $Q_n = \forall$ if $n$ is even, $Q_n = \exists$ if $n$ is odd%
~\cite[Def.\ 1.4.10]{Nies:2009}.
A set $A \subseteq \mathbb{N}$ is $\Pi^0_n$ iff
its complement $\mathbb{N} \setminus A$ is $\Sigma^0_n$.
We call the formula on the right hand side of \eqref{eq:def-Sigma^0_n} a
\emph{$\Sigma^0_n$-formula}, its negation is called \emph{$\Pi^0_n$-formula}.
It can be shown that
we can add any bounded quantifiers and
duplicate quantifiers of the same type
without changing the classification of $A$.
The set $A$ is $\Delta^0_n$ iff $A$ is $\Sigma^0_n$ and $A$ is $\Pi^0_n$.
We get that
$\Sigma^0_1$ as the class of recursively enumerable sets,
$\Pi^0_1$ as the class of co-recursively enumerable sets and
$\Delta^0_1$ as the class of recursive sets.

We say the set $A \subseteq \mathbb{N}$ is \emph{$\Sigma^0_n$-hard
($\Pi^0_n$-hard, $\Delta^0_n$-hard)} iff
for any set $B \in \Sigma^0_n$ ($B \in \Pi^0_n$, $B \in \Delta^0_n$),
$B$ is many-one reducible to $A$, i.e.,
there is a computable function $f$ such that
$k \in B \leftrightarrow f(k) \in A$~\cite[Def.\ 1.2.1]{Nies:2009}.
We get $\Sigma^0_n \subset \Delta^0_{n+1} \subset \Sigma^0_{n+1} \subset \ldots$
and $\Pi^0_n \subset \Delta^0_{n+1} \subset \Pi^0_{n+1} \subset \ldots$.
This hierarchy of subsets of natural numbers is known as
the \emph{arithmetical hierarchy}.

By Post's Theorem~\cite[Thm.\ 1.4.13]{Nies:2009},
a set is $\Sigma^0_n$ if and only if
it is recursively enumerable on an oracle machine
with an oracle for a $\Sigma^0_{n-1}$-complete set.

\subsection{Strings}

Let $\X$ be some finite set called \emph{alphabet}.
The set $\X^* := \bigcup_{n=0}^\infty \X^n$ is
the set of all finite strings over the alphabet $\X$,
the set $\X^\infty$ is
the set of all infinite strings
over the alphabet $\X$, and
the set $\X^\sharp := \X^* \cup \X^\infty$ is their union.
The empty string is denoted by $\epsilon$, not to be confused
with the small positive real number $\varepsilon$.
Given a string $x \in \X^*$, we denote its length by $|x|$.
For a (finite or infinite) string $x$ of length $\geq k$,
we denote with $x_{1:k}$ the first $k$ characters of $x$,
and with $x_{<k}$ the first $k - 1$ characters of $x$.
The notation $x_{1:\infty}$ stresses that $x$ is an infinite string.
We write $x \sqsubseteq y$ iff $x$ is a prefix of $y$, i.e., $x = y_{1:|x|}$.

\subsection{Computability of Real-valued Functions}

We fix some encoding of rational numbers into binary strings and
an encoding of binary strings into natural numbers.
From now on, this encoding will be done implicitly wherever necessary.

\begin{definition}[$\Sigma^0_n$-, $\Pi^0_n$-, $\Delta^0_n$-computable]
\label{def:computable}
A function $f: \X^* \to \mathbb{R}$ is called
\emph{$\Sigma^0_n$-computable ($\Pi^0_n$-computable, $\Delta^0_n$-computable)} iff
the set $\{ (x, q) \in \X^* \times \mathbb{Q} \mid f(x) > q \}$
is $\Sigma^0_n$ ($\Pi^0_n$, $\Delta^0_n$).
\end{definition}
A $\Delta^0_1$-computable function is called \emph{computable},
a $\Sigma^0_1$-computable function is called \emph{lower semicomputable}, and
a $\Pi^0_1$-computable function is called \emph{upper semicomputable}.
A $\Delta^0_2$-computable function $f$ is called \emph{limit computable},
because there is a computable function $\phi$ such that
\[
\lim_{k \to \infty} \phi(x, k) = f(x).
\]
The program $\phi$ that limit computes $f$
can be thought of as an \emph{anytime algorithm} for $f$:
we can stop $\phi$ at any time $k$ and get a preliminary answer.
If the program $\phi$ ran long enough (which we do not know),
this preliminary answer will be close to the correct one.

Limit-computable sets are the highest level in the arithmetical hierarchy
that can be approached by a regular Turing machine.
Above limit-computable sets
we necessarily need some form of halting oracle.
See \autoref{tab:computability} for the definition of
lower/upper semicomputable and limit-computable functions
in terms of the arithmetical hierarchy.

\begin{table}[t]
\begin{center}
\begin{tabular}{lcc}
& $f_>$ & $f_<$ \\
\hline
$f$ is computable            & $\Delta^0_1$ & $\Delta^0_1$ \\
$f$ is lower semicomputable  & $\Sigma^0_1$ & $\Pi^0_1$    \\
$f$ is upper semicomputable  & $\Pi^0_1$    & $\Sigma^0_1$ \\
$f$ is limit computable      & $\Delta^0_2$ & $\Delta^0_2$ \\
$f$ is $\Delta^0_n$-computable & $\Delta^0_n$ & $\Delta^0_n$ \\
$f$ is $\Sigma^0_n$-computable & $\Sigma^0_n$ & $\Pi^0_n$    \\
$f$ is $\Pi^0_n$-computable    & $\Pi^0_n$    & $\Sigma^0_n$ \\
\end{tabular}
\end{center}
\caption{
Connection between the computability of real-valued functions and
the arithmetical hierarchy.
We use the shorthand
$f_> := \{ (x, q) \mid f(x) > q \}$ and
$f_< := \{ (x, q) \mid f(x) < q \}$.
}
\label{tab:computability}
\end{table}

\begin{lemma}[Computability of Arithmetical Operations]
\label{lem:computable-reals}
Let $n > 0$ and
let $f, g: \X^* \to \mathbb{R}$ be two $\Delta^0_n$-computable functions.
Then
\begin{enumerate}[(i)]
\item $\{ (x, y) \mid f(x) >    g(y) \}$ is $\Sigma^0_n$,
\item $\{ (x, y) \mid f(x) \leq g(y) \}$ is $\Pi^0_n$,
\item $f + g$, $f - g$, and $f \cdot g$ are $\Delta^0_n$-computable, and
\item $f / g$ is $\Delta^0_n$-computable if $g(x) \neq 0$ for all $x$.
\end{enumerate}
\end{lemma}

\subsection{Algorithmic Information Theory}

A \emph{semimeasure} over the alphabet $\X$ is
a function $\nu: \X^* \to [0,1]$ such that
\begin{inparaenum}[(i)]
\item $\nu(\epsilon) \leq 1$, and
\item $\nu(x) \geq \sum_{a \in \X} \nu(xa)$ for all $x \in \X^*$.
\end{inparaenum}
A semimeasure is called (probability) \emph{measure} iff
for all $x$ equalities hold in (i) and (ii).
\emph{Solomonoff's prior $M$}~\cite{Solomonoff:1964} assigns to a string $x$
the probability that
the reference universal monotone Turing machine $U$~\cite[Ch.\ 4.5.2]{LV:2008}
computes a string starting with $x$
when fed with uniformly random bits as input.
The \emph{measure mixture $\MM$}~\cite[p.\ 74]{Gacs:1983}
removes the contribution of programs that do not compute infinite strings;
it is a measure except for a constant factor.
Formally,
\begin{align*}
M(x)   &:= \sum_{p:\, x \sqsubseteq U(p)} 2^{-|p|},
&
\MM(x) &:= \lim_{n \to \infty} \sum_{y \in \X^n} M(xy)
\end{align*}
Equivalently, the Solomonoff prior $M$ can be defined as
a mixture over all lower semicomputable semimeasures~\cite{WSH:2011}.
The function $M$ is a lower semicomputable semimeasure,
but not computable and not a measure~\cite[Lem.\ 4.5.3]{LV:2008}.
A semimeasure $\nu$ can be turned into a measure $\nu\norm$
using \emph{Solomonoff normalization}:
$\nu\norm(\epsilon) := 1$ and
for all $x \in \X^*$ and $a \in \X$,
\begin{equation}\label{eq:normalization}
   \nu\norm(xa)
:= \nu\norm(x) \frac{\nu(xa)}{\sum_{b \in \X} \nu(xb)}.
\end{equation}

\subsection{General Reinforcement Learning}
\label{ssec:general-rl}

In general reinforcement learning
the agent interacts with an environment in cycles:
at time step $t$ the agent chooses an \emph{action} $a_t \in \A$ and
receives a \emph{percept} $e_t = (o_t, r_t) \in \E$
consisting of an \emph{observation} $o_t \in \O$
and a real-valued \emph{reward} $r_t \in \mathbb{R}$;
the cycle then repeats for $t + 1$.
A \emph{history} is an element of $\H$.
We use $\ae \in \A \times \E$ to denote one interaction cycle,
and $\ae_{1:t}$ to denote a history of length $t$.
The goal in reinforcement learning is
to maximize total discounted rewards.
A \emph{policy} is a function $\pi: \H \to \A$
mapping each history to the action taken after seeing this history.

The environment can be stochastic,
but is assumed to be semicomputable.
In accordance with the AIXI literature~\cite{Hutter:2005},
we model environments as lower semicomputable
\emph{chronological conditional semimeasures} (LSCCCSs).
A \emph{conditional semimeasure} $\nu$ takes a sequence of actions $a_{1:\infty}$ as input
and returns a semimeasure $\nu(\,\cdot \dmid a_{1:\infty})$ over $\E^\sharp$.
A conditional semimeasure $\nu$ is \emph{chronological} iff
percepts at time $t$ do not depend on future actions, i.e.,
$\nu(e_{1:t} \dmid a_{1:\infty}) = \nu(e_{1:t} \dmid a'_{1:\infty})$
whenever $a_{1:t} = a'_{1:t}$.
Therefore we can write $\nu(e_{1:t} \dmid a_{1:t})$
instead of $\nu(e_{1:t} \dmid a_{1:\infty})$.
Despite their name,
conditional semimeasures do \emph{not} specify conditional probabilities;
the environment $\nu$ is \emph{not}
a joint probability distribution on actions and percepts.
Here we only care about the computability of the environment $\nu$;
for our purposes,
chronological conditional semimeasures behave just like semimeasures.

\subsection{The Universal Agent AIXI}
\label{ssec:aixi}

Our environment class $\M$ is the class of
all LSCCCSs.
Typically, Bayesian agents such as AIXI only function well
if the true environment is in their hypothesis class.
Since the hypothesis class $\M$ is extremely large,
the assumption that it contains the true environment is rather weak.
We fix the \emph{universal prior} $(w_\nu)_{\nu \in \M}$ with
$w_\nu > 0$ for all $\nu \in \M$ and $\sum_{\nu \in \M} w_\nu \leq 1$,
given by the reference machine $U$. 
The universal prior $w$ gives rise to
the \emph{universal mixture $\xi$},
which is a convex combination of all LSCCCSs $\M$:
\[
\xi(e_{<t} \dmid a_{<t}) := \sum_{\nu \in \M} w_\nu \nu(e_{<t} \dmid a_{<t})
\]
It is analogous to the Solomonoff prior $M$ but defined for reactive environments.
Like $M$, the universal mixture $\xi$ is lower semicomputable~\cite[Sec.\ 5.10]{Hutter:2005}.

We fix a \emph{discount function}
$\gamma: \mathbb{N} \to \mathbb{R}$ with
$\gamma_t := \gamma(t) \geq 0$ and $\sum_{t=1}^\infty \gamma_t < \infty$
and make the following assumptions.
\begin{assumption}\label{ass:aixi}
\begin{enumerate}[(a)]
\item \label{ass:discount-lsc}
	The discount function $\gamma$ is lower semicomputable.
\item \label{ass:bounded-rewards}
	Rewards are bounded between $0$ and $1$.
\item \label{ass:finite-actions-and-percepts}
	The set of actions $\A$ and the set of percepts $\E$ are both finite.
\end{enumerate}
\end{assumption}
\autoref{ass:aixi} (b) could be relaxed to bounded rewards
because we can rescale rewards $r \mapsto cr + d$ for any $c, d \in \mathbb{R}$
without changing optimal policies if the environment $\nu$ is a measure.
However, for our value-related results,
we require that rewards are nonnegative.

We define the \emph{discount normalization factor}
$\Gamma_t := \sum_{i=t}^\infty \gamma_i$.
There is no requirement that $\Gamma_t > 0$.
In fact, we use $\gamma$ for both,
AIXI with discounted infinite horizon ($\Gamma_t > 0$ for all $t$), and
AIXI with finite lifetime $m$.
In the latter case we set
\[
\gamma_{\mathrm{LT} m}(t) :=
\begin{cases}
1 &\text{if } t \leq m \\
0 &\text{if } t >    m.
\end{cases}
\]

If we knew the true environment $\nu \in \M$,
we would choose the $\nu$-optimal agent known as $\AINU$
that maximizes $\nu$-expected value (if $\nu$ is a measure).
Since we do not know the true environment,
we use the universal mixture $\xi$ over all environments in $\M$ instead.
This yields the Bayesian agent $\AIXI$:
it weighs every environment $\nu \in \M$
according to its prior probability $w_\nu$.

\begin{definition}[{Iterative Value Function~\cite[Def.\ 5.30]{Hutter:2005}}]
\label{def:V}
The \emph{value} of a policy $\pi$ in an environment $\nu$
given history $\ae_{<t}$ is
\[
   V^\pi_\nu(\ae_{<t})
:= \frac{1}{\Gamma_t} \lim_{m \to \infty} \sum_{e_{t:m}} R(e_{t:m})
     \nu( e_{1:m} \mid e_{<t} \dmid a_{1:m})
\]
if $\Gamma_t > 0$ and $V^\pi_\nu(\ae_{<t}) := 0$ if $\Gamma_t = 0$
where $a_i := \pi(e_{<i})$ for all $i \geq t$
and $R(e_{t:m}) := \sum_{k=t}^m \gamma_k r_k$.
The \emph{optimal value} is defined as
$V^*_\nu(h) := \sup_\pi V^\pi_\nu(h)$.
\end{definition}

Let $\ae_{<t} \in \H$ be some history.
We extend the value functions $V^\pi_\nu$
to include initial interactions
(in reinforcement learning literature on MDPs these are called $Q$-values),
$V^\pi_\nu(\ae_{<t} a_t) := V^{\pi'}_\nu(\ae_{<t})$
where $\pi'$ is the policy $\pi$ except that it takes action $a_t$ next, i.e.,
$\pi'(\ae_{<t}) := a_t$ and $\pi'(h) := \pi(h)$ for all $h \neq \ae_{<t}$.
We define $V^*_\nu(\ae_{<t} a_t) := \sup_\pi V^\pi_\nu(\ae_{<t} a_t)$
analogously.

\begin{definition}[{Optimal Policy~\cite[Def.\ 5.19 \& 5.30]{Hutter:2005}}]
\label{def:optimal-policy}
A policy $\pi$ is \emph{optimal in environment $\nu$ ($\nu$-optimal)} iff
for all histories
the policy $\pi$ attains the optimal value:
$V^\pi_\nu(h) = V^*_\nu(h)$
for all $h \in \H$.
\end{definition}

Since the discount function is summable, rewards are bounded
(\hyperref[ass:bounded-rewards]{\autoref*{ass:aixi}\ref*{ass:bounded-rewards}}), and
actions and percepts spaces are both finite
(\hyperref[ass:finite-actions-and-percepts]{\autoref*{ass:aixi}\ref*{ass:finite-actions-and-percepts}}),
an optimal policy exists
for every environment $\nu \in \M$~\cite[Thm.\ 10]{LH:2014discounting}.
For a fixed environment $\nu$,
an explicit expression for the optimal value function is
\begin{equation}\label{eq:V-explicit}
  V^*_\nu(\ae_{<t})
= \frac{1}{\Gamma_t} \lim_{m \to \infty} \expectimax{\ae_{t:m}}\;
    R(e_{t:m}) \nu(e_{1:m} \mid e_{<t} \dmid a_{1:m}),
\end{equation}
where $\expectimax{}$ denotes the expectimax operator:
\[
   \expectimax{\ae_{t:m}}
:= \max_{a_t \in \A} \sum_{e_t \in \E} \ldots \max_{a_m \in \A} \sum_{e_m \in \E}
\]

For an environment $\nu \in \M$ (an LSCCCS),
AINU is defined as a $\nu$-optimal policy
$\pi^*_\nu = \argmax_\pi V^\pi_\nu(\epsilon)$.
To stress that the environment is given by a measure $\mu \in \M$
(as opposed to a semimeasure),
we use AIMU.
AIXI is defined as a $\xi$-optimal policy $\pi^*_\xi$
for the universal mixture $\xi$~\cite[Ch.\ 5]{Hutter:2005}.
Since $\xi \in \M$ and every measure $\mu \in \M$ is also a semimeasure,
both $\AIMU$ and $\AIXI$ are a special case of $\AINU$.
However, $\AIXI$ is not a special case of $\AIMU$
since the mixture $\xi$ is not a measure.

Because there can be more than one optimal policy,
the definitions of AINU, AIMU and AIXI are not unique.
More specifically, a $\nu$-optimal policy maps a history $h$ to
\begin{equation}\label{eq:argmax}
\pi^*_\nu(h) :\in \argmax_{a \in \A} V^*_\nu(ha).
\end{equation}
If there are multiple actions $\alpha, \beta \in \A$ that attain the optimal value,
$V^*_\nu(h \alpha) = V^*_\nu(h \beta)$,
we say there is an \emph{argmax tie}.
Which action we settle on in case of a tie (how we break the tie)
is irrelevant and can be arbitrary.

\section{The Complexity of AINU, AIMU, and AIXI}
\label{sec:complexity-aixi}

\subsection{The Complexity of Solomonoff Induction}
\label{ssec:complexity-solomonoff-summary}

AIXI uses an analogue to Solomonoff's prior on all possible environments $\M$.
Therefore we first state computability results
for \emph{Solomonoff's prior $M$} and the \emph{measure mixture $\MM$}
in \autoref{tab:complexity-induction}~\cite{LH:2015computability2}.
Notably, $M$ is lower semicomputable and its conditional is limit computable.
However, neither the measure mixture $\MM$ nor any of its variants
are limit computable.

\begin{table}[t]
\begin{center}
\renewcommand{\arraystretch}{1.2}
\setlength{\tabcolsep}{5pt}
\begin{tabular}{llll}
           & Plain      & Conditional \\
\hline
$M$        & $\Sigma^0_1 \setminus \Delta^0_1$
           & $\Delta^0_2 \setminus (\Sigma^0_1 \cup \Pi^0_1)$ \\
$M\norm$   & $\Delta^0_2 \setminus (\Sigma^0_1 \cup \Pi^0_1)$
           & $\Delta^0_2  \setminus (\Sigma^0_1 \cup \Pi^0_1)$ \\
$\MM$      & $\Pi^0_2 \setminus \Delta^0_2$
           & $\Delta^0_3 \setminus (\Sigma^0_2 \cup \Pi^0_2)$ \\
$\MM\norm$ & $\Delta^0_3 \setminus (\Sigma^0_2 \cup \Pi^0_2)$
           & $\Delta^0_3 \setminus (\Sigma^0_2 \cup \Pi^0_2)$
\end{tabular}
\end{center}
\caption{
The complexity of the set
$\{ (x, q) \in \X^* \times \mathbb{Q} \mid f(x) > q \}$
where $f \in \{ M, M\norm, \MM, \MM\norm \}$
is one of the various versions of Solomonoff's prior.
Lower bounds on the complexity of $\MM$ and $\MM\norm$
hold only for specific universal Turing machines.
}
\label{tab:complexity-induction}
\end{table}

\subsection{Upper Bounds}
\label{ssec:upper-bounds}

In this section, we derive upper bounds
on the computability of AINU, AIMU, and AIXI.
Except for \autoref{cor:complexity-aimu},
all results in this section apply generally to
any LSCCCS $\nu \in \M$,
hence they apply to AIXI even though they are stated for AINU.

For a fixed lifetime $m$, only the first $m$ interactions matter.
There is a finite number of policies
that are different for the first $m$ interactions,
and the optimal policy $\pi^*_\xi$
can be encoded in a finite number of bits and is thus computable.
To make a meaningful statement about the computability of $\AINULT$,
we have to consider it as the function that takes the lifetime $m$
and outputs a policy $\pi^*_\xi$ that is optimal in the environment $\xi$
using the discount function $\gamma_{\textrm{LT} m}$.
In contrast,
for infinite lifetime discounting we just consider the function
$\pi^*_\xi: \H \to \A$.

In order to position AINU in the arithmetical hierarchy,
we need to identify these functions with sets of natural numbers.
In both cases, finite and infinite lifetime,
we represent these functions as relations
over $\mathbb{N} \times \H \times \A$ and $\H \times \A$ respectively.
These relations are easily identified with sets of natural numbers
by encoding the tuple with their arguments into one natural number.
From now on this translation of policies (and $m$)
into sets of natural numbers
will be done implicitly wherever necessary.

\begin{lemma}[Policies are in $\Delta^0_n$]
\label{lem:policy-is-delta_n}
If a policy $\pi$ is $\Sigma^0_n$ or $\Pi^0_n$,
then $\pi$ is $\Delta^0_n$.
\end{lemma}
\begin{proof}
Let $\varphi$ be a $\Sigma^0_n$-formula ($\Pi^0_n$-formula) defining $\pi$,
i.e., $\varphi(h, a)$ holds iff $\pi(h) = a$.
We define the formula $\varphi'$,
\[
   \varphi'(h, a)
:= \!\!\!\bigwedge_{a' \in \A \setminus \{ a \}}\!\!\! \neg \varphi(h, a').
\]
The set of actions $\A$ is finite,
hence $\varphi'$ is a $\Pi^0_n$-formula ($\Sigma^0_n$-formula).
Moreover, $\varphi'$ is equivalent to $\varphi$.
\end{proof}

To compute the optimal policy,
we need to compute the value function.
The following lemma gives an upper bound on
the computability of the value function for environments in $\M$.

\begin{lemma}[Complexity of $V^*_\nu$]
\label{lem:complexity-V}
For every LSCCCS $\nu \in \M$,
the function $V^*_\nu$ is $\Pi^0_2$-computable.
For $\gamma = \gamma_{\mathrm{LT} m}$
the function $V^*_\nu$ is $\Delta^0_2$-computable.
\end{lemma}
\begin{proof}
Multiplying \eqref{eq:V-explicit} with $\Gamma_t \nu(e_{<t} \dmid a_{<t})$
yields
$V^*_\nu(\ae_{<t}) > q$ if and only if
\begin{equation}\label{eq:V-multiplied}
  \lim_{m \to \infty} \expectimax{\ae_{t:m}}\;
    \nu(e_{1:m} \dmid a_{1:m}) R(e_{t:m})
> q\, \Gamma_t\, \nu(e_{<t} \dmid a_{<t}).
\end{equation}
The inequality's right side is lower semicomputable,
hence there is a computable function $\psi$ such that
$\psi(\ell) \nearrow q\, \Gamma_t\, \nu(e_{<t} \dmid a_{<t}) =: q'$
for $\ell \to \infty$.
For a fixed $m$, the left side is also lower semicomputable,
therefore there is a computable function $\phi$ such that
$\phi(m, k) \nearrow \expectimax{\ae_{t:m}} \nu(e_{1:m} \dmid a_{1:m}) R(e_{t:m}) =: f(m)$
for $k \to \infty$.
We already know that the limit of $f(m)$ for $m \to \infty$ exists (uniquely),
hence we can write \eqref{eq:V-multiplied} as
\begin{align*}
                      &\lim_{m \to \infty} f(m) > q' \\
\Longleftrightarrow~~ &\forall m_0\, \exists m \geq m_0.\; f(m) > q' \\
\Longleftrightarrow~~ &\forall m_0\, \exists m \geq m_0\, \exists k.\;
                         \phi(m, k) > q' \\
\Longleftrightarrow~~ &\forall \ell\, \forall m_0\, \exists m \geq m_0\,
                         \exists k.\; \phi(m, k) > \psi(\ell),
\end{align*}
which is a $\Pi^0_2$-formula.
In the finite lifetime case where
$m$ is fixed,
the value function $V^*_\nu(\ae_{<t})$
is $\Delta^0_2$-computable by \autoref{lem:computable-reals} (iv),
since $V^*_\nu(\ae_{<t}) = f(m)q / q'$.
\qedhere
\end{proof}

From the optimal value function $V^*_\nu$
we get the optimal policy $\pi^*_\nu$ according to \eqref{eq:argmax}.
However, in cases where there is more than one optimal action,
we have to break an argmax tie.
This happens iff $V^*_\nu(h\alpha) = V^*_\nu(h\beta)$
for two potential actions $\alpha \neq \beta \in \A$.
This equality test is more difficult than
determining which is larger in cases where they are unequal.
Thus we get the following upper bound.

\begin{theorem}[Complexity of Optimal Policies]
\label{thm:complexity-optimal-policies}
For any environment $\nu \in \M$,
if $V^*_\nu$ is $\Delta^0_n$-computable,
then there is an optimal policy $\pi^*_\nu$ for the environment $\nu$
that is $\Delta^0_{n+1}$.
\end{theorem}
\begin{proof}
To break potential ties,
we pick an (arbitrary) total order $\succ$ on $\A$
that specifies which actions should be preferred in case of a tie.
We define
\begin{equation}\label{eq:p*}
\begin{aligned}
\pi_\nu(h) = a
~:\Longleftrightarrow~~~~~~~
&\bigwedge_{a': a' \succ a}
  V^*_\nu(ha) > V^*_\nu(ha') \\
\;\land &\bigwedge_{a': a \succ a'}
  V^*_\nu(ha) \geq V^*_\nu(ha').
\end{aligned}
\end{equation}
Then $\pi_\nu$ is a $\nu$-optimal policy according to \eqref{eq:argmax}.
By assumption, $V^*_\nu$ is
$\Delta^0_n$-computable.
By \autoref{lem:computable-reals} (i) and (ii)
$V^*_\nu(ha) > V^*_\nu(ha')$ is in $\Sigma^0_n$ and
$V^*_\nu(ha) \geq V^*_\nu(ha')$ is $\Pi^0_n$.
Therefore the policy $\pi_\nu$ defined in \eqref{eq:p*} is a conjunction
of a $\Sigma^0_n$-formula and a $\Pi^0_n$-formula and thus in $\Delta^0_{n+1}$.
\end{proof}

\begin{corollary}[Complexity of AINU]
\label{cor:complexity-ainu}
$\AINULT$ is $\Delta^0_3$ and
$\AINUDC$ is $\Delta^0_4$
for every environment $\nu \in \M$.
\end{corollary}
\begin{proof}
From \autoref{lem:complexity-V} and \autoref{thm:complexity-optimal-policies}.
\end{proof}

Usually we do not mind taking slightly suboptimal actions.
Therefore actually trying to determine if two actions have the exact same value
seems like a waste of resources.
In the following, we consider policies that attain a value
that is always within some $\varepsilon > 0$ of the optimal value.

\begin{definition}[$\varepsilon$-Optimal Policy]
\label{def:eps-optimal-policy}
A policy $\pi$ is \emph{$\varepsilon$-optimal in environment $\nu$} iff
$V^*_\nu(h) - V^\pi_\nu(h) < \varepsilon$ for all histories $h \in \H$.
\end{definition}

\begin{theorem}[Complexity of $\varepsilon$-Optimal Policies]
\label{thm:complexity-eps-optimal-policies}
For any environment $\nu \in \M$,
if $V^*_\nu$ is $\Delta^0_n$-computable,
then there is an $\varepsilon$-optimal policy $\pi^\varepsilon_\nu$
for the environment $\nu$
that is $\Delta^0_n$.
\end{theorem}
\begin{proof}
Let $\varepsilon > 0$ be given.
Since the value function $V^*_\nu(h)$ is $\Delta^0_n$-computable,
the set $V_\varepsilon := \{ (ha, q) \mid |q - V^*_\nu(ha)| < \varepsilon / 2 \}$
is in $\Delta^0_n$ according to \autoref{def:computable}.
Hence we compute the values $V^*_\nu(ha')$
until we get within $\varepsilon / 2$ for every $a' \in \A$
and then choose the action with the highest value so far.
Formally,
let $\succ$ be an arbitrary total order on $\A$
that specifies which actions should be preferred in case of a tie.
Without loss of generality, we assume $\varepsilon = 1/k$,
and define $Q$ to be an $\varepsilon / 2$-grid on $[0, 1]$,
i.e., $Q := \{ 0, 1/2k, 2/2k, \ldots, 1 \}$.
We define
\begin{equation}\label{eq:eps-optimal-policy}
\begin{aligned}
\pi^\varepsilon_\nu(h) = a
:\Longleftrightarrow
\exists (q_{a'})_{a' \in \A} \in Q^{\A}.\;
          &\bigwedge_{a' \in \A} (ha', q_{a'}) \in V_\varepsilon \\
\land &\bigwedge_{a': a' \succ a} q_a > q_{a'}
    \land \bigwedge_{a': a \succ a'} q_a \geq q_{a'} \\
\land &\text{ the tuple $(q_{a'})_{a' \in \A}$ is minimal with} \\
&\text{ respect to
the lex.\ ordering on $Q^\A$}.
\end{aligned}
\end{equation}
This makes the choice of $a$ unique.
Moreover, $Q^\A$ is finite since $\A$ is finite,
and hence \eqref{eq:eps-optimal-policy} is a $\Delta^0_n$-formula.
\end{proof}

\begin{corollary}[Complexity of $\varepsilon$-Optimal AINU]
\label{cor:complexity-eps-ainu}
For any environment $\nu \in \M$,
there is an $\varepsilon$-optimal policy for $\AINULT$
that is $\Delta^0_2$ and
there is an $\varepsilon$-optimal policy for $\AINUDC$
that is $\Delta^0_3$.
\end{corollary}
\begin{proof}
From \autoref{lem:complexity-V} and \autoref{thm:complexity-eps-optimal-policies}.
\end{proof}

If the environment $\nu \in \M$ is a measure,
i.e., $\nu$ assigns zero probability to finite strings,
then we get computable $\varepsilon$-optimal policies.

\begin{corollary}[Complexity of AIMU]
\label{cor:complexity-aimu}
If the environment $\mu \in \M$ is a measure and
the discount function $\gamma$ is computable,
then and $\AIMULT$ and $\AIMUDC$ are limit computable ($\Delta^0_2$), and
$\varepsilon$-optimal $\AIMULT$ and $\AIMUDC$ are computable ($\Delta^0_1$).
\end{corollary}
\begin{proof}
In the discounted case,
we can truncate the limit $m \to \infty$ in \eqref{eq:V-explicit}
at the $\varepsilon / 2$-effective horizon
$m_{\text{eff}} := \min \{ k \mid \Gamma_k / \Gamma_t < \varepsilon/2 \}$,
since everything after $m_{\text{eff}}$ can contribute at most $\varepsilon/2$
to the value function.
Any lower semicomputable measure is computable~\cite[Lem.\ 4.5.1]{LV:2008}.
Therefore $V^*_\mu$ as given in \eqref{eq:V-explicit}
is composed only of computable functions,
hence it is computable according to \autoref{lem:computable-reals}.
The claim now follows from \autoref{thm:complexity-optimal-policies} and
\autoref{thm:complexity-eps-optimal-policies}.
\end{proof}

\subsection{Lower Bounds}
\label{ssec:lower-bounds}

We proceed to show that the bounds from the previous section
are the best we can hope for.
In environment classes where ties have to be broken,
$\AIMUDC$ has to solve $\Sigma^0_3$-hard problems
(\autoref{thm:AINU_DC-is-Sigma3-hard}), and
$\AIMULT$ has to solve $\Pi^0_2$-hard problems
(\autoref{thm:AINU_LT-is-Pi2-hard}).
These lower bounds are stated for particular environments $\nu \in \M$.

We also construct universal mixtures
that yield bounds on $\varepsilon$-optimal policies.
In the finite lifetime case, there is an $\varepsilon$-optimal $\AIXILT$
that solves $\Sigma^0_1$-hard problems
(\autoref{thm:eps-AIXI_LT-is-Sigma1-hard}), and
for general discounting,
there is an $\varepsilon$-optimal $\AIXIDC$ that solves $\Pi^0_2$-hard problems
(\autoref{thm:eps-AIXI_DC-is-Pi2-hard}).
For arbitrary universal mixtures,
we prove the following weaker statement that only guarantees incomputability.

\begin{theorem}[No $\AIXI$ is computable]
\label{thm:AIXI-is-not-computable}
$\AIXI_\textrm{LT}$ and $\AIXI_\textrm{DC}$ are not computable
for any universal Turing machine $U$.
\end{theorem}

This theorem follows from the incomputability of Solomonoff induction.
Since AIXI uses an analogue of Solomonoff's prior for learning,
it succeeds to predict the environment's behavior
for its own policy~\cite[Thm.\ 5.31]{Hutter:2005}.
If AIXI were computable,
then there would be computable environments more powerful than AIXI:
they can simulate AIXI and anticipate its prediction,
which leads to a contradiction.

\begin{proof}
Assume there is a computable policy $\pi^*_\xi$ that is optimal in $\xi$.
We define a deterministic environment $\mu$,
the \emph{adversarial environment} to $\pi^*_\xi$.
The environment $\mu$
gives rewards $0$ as long as the agent follows the policy $\pi^*_\xi$,
and rewards $1$ once the agent deviates.
Formally, we ignore observations by setting $\O := \{ 0 \}$, and define
\begin{align*}
\mu(r_{1:t} \dmid a_{1:t}) :=
\begin{cases}
1 &\text{if } \forall k \leq t.\, a_k = \pi^*_\xi((ar)_{<k})
   \text{ and } r_k = 0 \\
1 &\text{if } \forall k \leq t.\, r_k = \one_{k \geq i} \\
  &\text{ where } i := \min \{ j \mid a_j \neq \pi^*_\xi((ar)_{<j}) \} \\
0 &\text{otherwise}.
\end{cases}
\end{align*}
The environment $\mu$ is computable because
the policy $\pi^*_\xi$ was assumed to be computable.
Suppose $\pi^*_\xi$ acts in $\mu$, then by \cite[Thm.\ 5.36]{Hutter:2005},
AIXI learns to predict perfectly \emph{on policy}:
\[
    V^*_\xi(\ae_{<t})
=   V^{\pi^*_\xi}_\xi(\ae_{<t})
\to V^{\pi^*_\xi}_\mu(\ae_{<t})
=   0 \text{ as } t \to \infty,
\]
since both $\pi^*_\xi$ and $\mu$ are deterministic.
Therefore we find a $t$ large enough such that
$V^*_\xi(\ae_{<t}) < w_\mu$
(in the finite lifetime case we choose $m > t$)
where $\ae_{<t}$ is the interaction history of $\pi^*_\xi$ in $\mu$.
A policy $\pi$ with $\pi(\ae_{<t}) \neq \pi^*_\xi(\ae_{<t})$,
gets a reward of $1$ in environment $\mu$ for all time steps after $t$,
hence $V^\pi_\mu(\ae_{<t}) = 1$.
With linearity of $V^\pi_\xi(\ae_{<t})$ in $\xi$~\cite[Thm.\ 5.31]{Hutter:2005},
\[
     V^\pi_\xi(\ae_{<t})
\geq w_\mu \tfrac{\mu(e_{1:t} \dmid a_{1:t})}{\xi(e_{1:t} \dmid a_{1:t})}
       V^\pi_\mu(\ae_{<t})
\geq w_\mu,
\]
since $\mu(e_{1:t} \dmid a_{1:t}) = 1$ ($\mu$ is deterministic),
$V^\pi_\mu(\ae_{<t}) = 1$, and $\xi(e_{1:t} \dmid a_{1:t}) \leq 1$.
Now we get a contradiction:
\[
     w_\mu
>    V^*_\xi(\ae_{<t}) \\
=    \max_{\pi'} V^{\pi'}_\xi(\ae_{<t})
\geq V^{\pi}_\xi(\ae_{<t})
\geq w_\mu
\qedhere
\]
\end{proof}

For the remainder of this section,
we fix the action space to be $\A := \{ \alpha, \beta \}$
with action $\alpha$ favored in ties.
The percept space is fixed to a tuple of binary observations and rewards,
$\E := \O \times \{ 0, 1 \}$ with $\O := \{ 0, 1 \}$.

\begin{theorem}[$\AINUDC$ is $\Sigma^0_3$-hard]
\label{thm:AINU_DC-is-Sigma3-hard}
If $\Gamma_t > 0$ for all $t$,
there is an environment $\nu \in \M$ such that
$\AINUDC$ is $\Sigma^0_3$-hard.
\end{theorem}
\begin{proof}
Let $A$ be any $\Sigma^0_3$ set, then
there is a computable relation $S$ such that
\begin{equation}\label{eq:defA}
n \in A
\;\Longleftrightarrow\;
\exists i\; \forall t\; \exists k\; S(n, i, t, k).
\end{equation}
We define a class of environments
$\M' = \{ \rho_0, \rho_1, \ldots \} \subset \M$
where each environment $\rho_i$ is defined by
\begin{align*}
\rho_i((or)_{1:t} \dmid a_{1:t}) :=
\begin{cases}
2^{-t},   &\text{if } o_{1:t} = 1^t
             \text{ and } \forall t' \leq t.\; r_{t'} = 0 \\
2^{-n-1}, &\text{if } \exists n.\; 1^n0 \sqsubseteq o_{1:t} \sqsubseteq 1^n 0^\infty
             \text{ and } a_{n+2} = \alpha \\
            &\text{ and } \forall t' \leq t.\; r_{t'} = 0 \\
2^{-n-1}, &\text{if } \exists n.\; 1^n0 \sqsubseteq o_{1:t} \sqsubseteq 1^n 0^\infty
             \text{ and } a_{n+2} = \beta \\
            &\text{ and } \forall t' \leq t.\; r_{t'} = \one_{t' > n+1} \\
            &\text{ and } \forall t' \leq t\, \exists k\; S(n, i, t', k) \\
0,          &\text{otherwise}.
\end{cases}
\end{align*}
Every $\rho_i$ is a chronological conditional semimeasure by definition,
so $\M' \subseteq \M$.
Furthermore,
every $\rho_i$ is lower semicomputable since $S$ is computable.

We define our environment $\nu$ as a mixture over $\M'$,
\[
\nu := \sum_{i \in \mathbb{N}} 2^{-i-1} \rho_i;
\]
the choice of the weights on the environments $\rho_i$ is arbitrary but positive.
Let $\pi^*_\nu$ be an optimal policy for the environment $\nu$ and
recall that the action $\alpha$ is preferred in ties.
We claim that for the $\nu$-optimal policy $\pi^*_\nu$,
\begin{equation}\label{eq:AINU_DC-is-Sigma3-hard-claim}
n \in A
~~\Longleftrightarrow~~
\pi^*_\nu(1^n0) = \beta.
\end{equation}
This enables us to decide whether $n \in A$ given the policy $\pi^*_\nu$,
hence proving \eqref{eq:AINU_DC-is-Sigma3-hard-claim} concludes this proof.

Let $n, i \in \mathbb{N}$ be given, and
suppose we are in environment $i$ and observe $1^n 0$.
Taking action $\alpha$ next yields rewards $0$ forever;
taking action $\beta$ next yields
a reward of $1$ for those time steps $t \geq n+2$ for which
$\forall t' \leq t\, \exists k\; S(n, i, t', k)$,
after that the semimeasure assigns probability $0$ to all next observations.
Therefore, if for some $t_0$ there is no $k$ such that $S(n, i, t_0, k)$,
then $\rho_i(e_{1:t_0} \dmid \ldots \beta \ldots) = 0$, and hence
\[
V^*_{\rho_i}(1^n 0 \beta) = 0 = V^*_{\rho_i}(1^n 0 \alpha),
\]
and otherwise $\rho_i$ yields reward $1$ for every time step after $n + 1$,
therefore
\[
V^*_{\rho_i}(1^n 0 \beta) = \Gamma_{n+2} > 0 = V^*_{\rho_i}(1^n 0 \alpha)
\]
(omitting the first $n + 1$ actions and rewards
in the argument of the value function).
We can now show \eqref{eq:AINU_DC-is-Sigma3-hard-claim}:
By \eqref{eq:defA}, $n \in A$ if and only if there is an $i$
such that for all $t$ there is a $k$ such that $S(n, i, t, k)$,
which happens if and only if
there is an $i \in \mathbb{N}$ such that $V^*_{\rho_i}(1^n 0 \beta) > 0$,
which is equivalent to $V^*_\nu(1^n 0 \beta) > 0$,
which in turn is equivalent to $\pi^*_\mu(1^n 0) = \beta$
since $V^*_\nu(1^n 0 \alpha) = 0$ and action $\alpha$ is favored in ties.
\end{proof}

\begin{theorem}[$\AINULT$ is $\Pi^0_2$-hard]
\label{thm:AINU_LT-is-Pi2-hard}
There is an environment $\nu \in \M$ such that
$\AINULT$ is $\Pi^0_2$-hard.
\end{theorem}

The proof of \autoref{thm:AINU_LT-is-Pi2-hard} is analogous to
the proof of \autoref{thm:AINU_DC-is-Sigma3-hard}.
The notable difference is that
we replace $\forall t' \leq t\; \exists k\; S(n, i, t', k)$
with $\exists k\; S(n, i, k)$.
Moreover, we swap actions $\alpha$ and $\beta$:
action $\alpha$ `checks' the relation $S$ and
action $\beta$ gives a sure reward of $1$.

\begin{theorem}[Some $\varepsilon$-optimal $\AIXI_\textrm{LT}$ are $\Sigma^0_1$-hard]
\label{thm:eps-AIXI_LT-is-Sigma1-hard}
There is a universal Turing machine $U'$ and an $\varepsilon > 0$ such that
any $\varepsilon$-optimal policy for $\AIXI_\textrm{LT}$ is $\Sigma^0_1$-hard.
\end{theorem}
\begin{proof}
Let $\xi$ denote the universal mixture
derived from the reference universal monotone Turing machine $U$.
Let $A$ be a $\Sigma^0_1$-set and $S$ computable relation such that
$n+1 \in A$ iff $\exists k\; S(n, k)$.
We define the environment
\begin{align*}
\nu((or)_{1:t} \dmid a_{1:t}) :=
\begin{cases}
\xi((or)_{1:n} \dmid a_{1:n}), &\text{if } \exists n.\; o_{1:n} = 1^{n-1}0 \\
                                   &\text{ and } a_n = \alpha \\
                                   &\text{ and } \forall t' > n.\; e_{t'} = (0, \tfrac{1}{2}) \\
\xi((or)_{1:n} \dmid a_{1:n}), &\text{if } \exists n.\; o_{1:n} = 1^{n-1}0 \\
                                   &\text{ and } a_n = \beta \\
                                   &\text{ and } \forall t' > n.\; e_t = (0, 1) \\
                                   &\text{ and } \exists k\; S(n-1, k). \\
\xi((or)_{1:t} \dmid a_{1:t}),     &\text{if } \nexists n.\; o_{1:n} = 1^{n-1}0 \\
0,                                 &\text{otherwise}.
\end{cases}
\end{align*}
The environment $\nu$ mimics the universal environment $\xi$
until the observation history is $1^{n-1}0$.
Taking the action $\alpha$ next gives rewards $1/2$ forever.
Taking the action $\beta$ next gives rewards $1$ forever if $n \in A$,
otherwise the environment $\nu$ ends at some future time step.
Therefore we want to take action $\beta$ if and only if $n \in A$.
We have that $\nu$ is an LSCCCS since
$\xi$ is an LSCCCS and $S$ is computable.

We define
the universal lower semicomputable semimeasure
$\xi' := \tfrac{1}{2} \nu + \tfrac{1}{8} \xi$.
Choose $\varepsilon := 1 / 9$.
Let $n \in A$ be given and define the lifetime $m := n + 1$.
Let $h \in (\A \times \E)^n$ be any history
with observations $o_{1:n} = 1^{n-1}0$.
Since $\nu(1^{n-1}0 \mid a_{1:n}) = \xi(1^{n-1}0 \mid a_{1:n})$ by definition,
the posterior weights of $\nu$ and $\xi$ in $\xi'$
are equal to the prior weights,
analogously to \cite[Thm.\ 7]{LH:2015priors}.
In the following,
we use the linearity of
$V^{\pi^*_{\xi'}}_\rho$ in $\rho$~\cite[Thm.\ 5.21]{Hutter:2005},
and the fact that values are bounded between $0$ and $1$.
If there is a $k$ such that $S(n-1, k)$,
\begin{align*}
      V^*_{\xi'}(h\beta) - V^*_{\xi'}(h\alpha)
&=    \tfrac{1}{2} V^{\pi^*_{\xi'}}_\nu(h\beta)
      - \tfrac{1}{2} V^{\pi^*_{\xi'}}_\nu(h\alpha)
      + \tfrac{1}{8} V^{\pi^*_{\xi'}}_\xi(h\beta)
      - \tfrac{1}{8} V^{\pi^*_{\xi'}}_\xi(h\alpha) \\
&\geq \tfrac{1}{2} - \tfrac{1}{4} + 0 - \tfrac{1}{8}
=     \tfrac{1}{8},
\end{align*}
and similarly if there is no $k$ such that $S(n-1, k)$, then
\begin{align*}
      V^*_{\xi'}(h\alpha) - V^*_{\xi'}(h\beta)
&=    \tfrac{1}{2} V^{\pi^*_{\xi'}}_\nu(h\alpha)
      - \tfrac{1}{2} V^{\pi^*_{\xi'}}_\nu(h\beta)
      + \tfrac{1}{8} V^{\pi^*_{\xi'}}_\xi(h\alpha)
      - \tfrac{1}{8} V^{\pi^*_{\xi'}}_\xi(h\beta)\\
&\geq \tfrac{1}{4} - 0 + 0 - \tfrac{1}{8}
=     \tfrac{1}{8}.
\end{align*}
In both cases $|V^*_{\xi'}(h\beta) - V^*_{\xi'}(h\alpha)| > 1 / 9$.
Hence we pick $\varepsilon := 1/9$ and get for every
$\varepsilon$-optimal policy $\pi^\varepsilon_{\xi'}$ that
$\pi^\varepsilon_{\xi'}(h) = \beta$ if and only if $n \in A$.
\end{proof}

\begin{theorem}[Some $\varepsilon$-optimal $\AIXI_\textrm{DC}$ are $\Pi^0_2$-hard]
\label{thm:eps-AIXI_DC-is-Pi2-hard}
There is a universal Turing machine $U'$ and an $\varepsilon > 0$ such that
any $\varepsilon$-optimal policy for $\AIXI_\textrm{DC}$ is $\Pi^0_2$-hard.
\end{theorem}

The proof of \autoref{thm:eps-AIXI_DC-is-Pi2-hard} is analogous to
the proof of \autoref{thm:eps-AIXI_LT-is-Sigma1-hard}
except that we choose $\forall m' \leq m\, \exists k\; S(x, m, k)$
as a condition for reward $1$ after playing action $\beta$.

\section{Iterative vs.\ Recursive AINU}
\label{sec:recursive-ainu}

Generally, our environment $\nu \in \M$ is only a semimeasure and not a measure.
I.e., there is a history $\ae_{<t}a_t$ such that
\[
1 > \sum_{e_t \in \E} \nu(e_t \mid e_{<t} \dmid a_{1:t}).
\]
In such cases, with positive probability the environment $\nu$
does not produce a new percept $e_t$.
If this occurs, we shall use the informal interpretation that
the environment $\nu$ \emph{ended},
but our formal argument does not rely on this interpretation.

The following proposition shows that
for a semimeasure $\nu \in \M$ that is not a measure,
the iterative definition of AINU does not maximize $\nu$-expected rewards.
Recall that $\gamma_1$ states the discount of the first reward.
In the following, we assume without loss of generality that
$\gamma_1 > 0$, i.e.,
we are not indifferent about the reward received in time step $1$.

\begin{proposition}[Iterative AINU is not a $\nu$-Expected Rewards Maximizer]
\label{prop:AINU-is-not-a-reward-maximizer}
For any $\varepsilon > 0$
there is an environment $\nu \in \M$ that is not a measure and
a policy $\pi$ that receives a total of $\gamma_1$ rewards in $\nu$,
but AINU receives only $\varepsilon \gamma_1$ rewards in $\nu$.
\end{proposition}

Informally, the environment $\nu$ is defined as follows.
In the first time step,
the agent chooses between the two actions $\alpha$ and $\beta$.
Taking action $\alpha$ gives a reward of $1$,
and subsequently the environment ends.
Action $\beta$ gives a reward of $\varepsilon$,
but the environment continues forever.
There are no other rewards in this environment.
From the perspective of $\nu$-expected reward maximization,
it is better to take action $\alpha$,
however AINU takes action $\beta$.

\begin{proof}
Let $\varepsilon > 0$.
We ignore observations and
set $\E := \{ 0, \varepsilon, 1 \}$, $\A := \{ \alpha, \beta \}$.
The environment $\nu$ is formally defined by
\begin{align*}
\nu(r_{1:t} \dmid a_{1:t}) :=
\begin{cases}
1 &\text{if } a_1 = \alpha
   \text{ and } r_1 = 1
   \text{ and } t = 1 \\
1 &\text{if } a_1 = \beta
   \text{ and } r_1 = \varepsilon
   \text{ and } r_k = 0\; \forall 1 < k \leq t \\
0 &\text{otherwise}.
\end{cases}
\end{align*}
Taking action $\alpha$ first,
we have $\nu(r_{1:t} \dmid \alpha a_{2:t}) = 0$ for $t > 1$
(the environment $\nu$ ends in time step $2$ given history $\alpha$).
Hence we use \eqref{eq:V-explicit} to conclude
\[
  V^*_\nu(\alpha)
= \frac{1}{\Gamma_t} \lim_{m \to \infty} \sum_{r_{1:m}}
     \nu( r_{1:m} \dmid \alpha a_{2:m}) \sum_{i=1}^m r_i
= 0.
\]
Taking action $\beta$ first we get
\[
  V^*_\nu(\beta)
= \frac{1}{\Gamma_t} \lim_{m \to \infty} \sum_{r_{1:m}}
     \nu( r_{1:m} \dmid \beta a_{2:m}) \sum_{i=1}^m r_i
= \frac{\gamma_1}{\Gamma_1} \varepsilon.
\]
Since $\gamma_1 > 0$ and $\varepsilon > 0$,
we have $V^*_\nu(\beta) > V^*_\nu(\alpha)$,
and thus AIMU will use a policy that plays action $\beta$ first,
receiving a total discounted reward of $\varepsilon \gamma_1$.
In contrast,
any policy $\pi$ that takes action $\alpha$ first
receives a larger total discounted reward of $\gamma_1$.
\end{proof}

Whether it is reasonable to assume
that our environment has a nonzero probability of ending
is a philosophical debate we do not want to engage in here.
Instead,
we have a different motivation to use the recursive value function:
we get an improved computability result.
Concretely,
we show that for all environments $\nu \in \M$,
there is a limit-computable $\varepsilon$-optimal policy
maximizing $\nu$-expected rewards using an infinite horizon.
According to \autoref{thm:eps-AIXI_DC-is-Pi2-hard},
this does not hold for all $V^*_\nu$-maximizing agents AINU.

In order to maximize $\nu$-expected rewards in case $\nu$ is not a measure,
we need the recursive definition of the value function
(analogously to \cite[Eq.\ 4.12]{Hutter:2005}).
To avoid confusion, we denote it $W^\pi_\nu$:
\begin{align*}
  W^\pi_\nu(\ae_{<t})
= \frac{1}{\Gamma_t} \sum_{e_t} &\big( \gamma_t r_t
  + \Gamma_{t+1} W^\pi_\nu(\ae_{1:t}) \big) \nu(e_t \mid e_{<t} \dmid a_{1:t})
\end{align*}
where $a_t := \pi(\ae_{<t})$.
In the following we write it in non-recursive form.

\begin{definition}[{$\nu$-Expected Value Function}]
\label{def:W}
The \emph{$\nu$-expected value} of a policy $\pi$ in an environment $\nu$
given history $\ae_{<t}$ is
\[
   W^\pi_\nu(\ae_{<t})
:= \frac{1}{\Gamma_t} \sum_{m=t}^\infty \sum_{e_{t:m}} \gamma_m r_m
     \nu( e_{1:m} \mid e_{<t} \dmid a_{1:m})
\]
if $\Gamma_t > 0$ and $W^\pi_\nu(\ae_{<t}) := 0$ if $\Gamma_t = 0$
where $a_i := \pi(e_{<i})$ for all $i \geq t$.
The \emph{optimal $\nu$-expected value} is defined as
$W^*_\nu(h) := \sup_\pi W^\pi_\nu(h)$.
\end{definition}

The difference between $V^\pi_\nu$ and $W^\pi_\nu$ is that
for $W^\pi_\nu$ all obtained rewards matter,
but for $V^\pi_\nu$ only the rewards in timelines that continue indefinitely.
In this sense the value function $V^\pi_\nu$ conditions on surviving forever.
If the environment $\mu$ is a measure,
then the history is infinite with probability one,
and so $V^\pi_\nu$ and $W^\pi_\nu$ coincide.
Hence this distinction is not relevant for $\AIMU$,
only for $\AINU$ and $\AIXI$.

So why use $V^\pi_\nu$ in the first place?
Historically, this is how infinite-horizon AIXI has been defined%
~\cite[Def.\ 5.30]{Hutter:2005}.
This definition is the natural adaptation
of (optimal) minimax search in zero-sum games
to the (optimal) expectimax algorithm for stochastic environments.
It turns out to be problematic only because semimeasures have
positive probability of ending prematurely.

\begin{lemma}[Complexity of $W^*_\nu$]
\label{lem:complexity-W}
For every LSCCCS $\nu \in \M$,
and every lower semicomputable discount function $\gamma$,
the function $W^*_\nu$ is $\Delta^0_2$-computable.
\end{lemma}
\begin{proof}
The proof is analogous to the proof of \autoref{lem:complexity-V}.
We expand \autoref{def:W}
using the expectimax operator analogously to \eqref{eq:V-explicit}.
This gives a quotient with numerator
\[
\lim_{m \to \infty} \expectimax{\ae_{t:m}}\;
  \sum_{i=t}^m \gamma_i r_i \nu(e_{1:i} \dmid a_{1:i}),
\]
and denominator $\nu(e_{<t} \dmid a_{<t}) \cdot \Gamma_t$.
In contrast to the iterative value function,
the numerator is now nondecreasing in $m$
because we assumed rewards to be nonnegative
(\hyperref[ass:bounded-rewards]{\autoref*{ass:aixi}\ref*{ass:bounded-rewards}}).
Hence both numerator and denominator are lower semicomputable functions,
so \autoref{lem:computable-reals} (iv) implies that
$W^*_\nu$ is $\Delta^0_2$-computable.
\end{proof}

Now we can apply our results from \autoref{ssec:upper-bounds}
to show that using the recursive value function $W^\pi_\nu$,
we get a universal AI model with an infinite horizon
whose $\varepsilon$-approximation is limit computable.
Moreover, in contrast to iterative $\AINU$,
recursive $\AINU$ actually maximizes $\nu$-expected rewards.

\begin{corollary}[Complexity of Recursive AINU/AIXI]
\label{cor:complexity-recursive-ainu}
For any environment $\nu \in \M$,
recursive $\AINU$ is $\Delta^0_3$ and
there is an $\varepsilon$-optimal recursive $\AINU$ that is $\Delta^0_2$.
In particular,
for any universal Turing machine,
recursive $\AIXI$ is $\Delta^0_3$ and
there is an $\varepsilon$-optimal recursive $\AIXI$ that is limit computable.
\end{corollary}
\begin{proof}
From \autoref{thm:complexity-optimal-policies},
\autoref{thm:complexity-eps-optimal-policies}, and
\autoref{lem:complexity-W}.
\end{proof}

Analogously to \autoref{thm:AIXI-is-not-computable},
\autoref{thm:AINU_LT-is-Pi2-hard}, and \autoref{thm:eps-AIXI_LT-is-Sigma1-hard}
we can show that
recursive AIXI is not computable,
recursive AINU is $\Pi^0_2$-hard,
and for some universal Turing machines,
$\varepsilon$-optimal recursive AIXI is $\Sigma^0_1$-hard.

\section{Discussion}
\label{sec:discussion}

We set out with the goal of finding a limit-computable perfect agent.
\autoref{tab:complexity-induction} on page~\pageref*{tab:complexity-induction}
summarizes our computability results
regarding Solomonoff's prior $M$:
conditional $M$ and $M\norm$ are limit computable,
while $\MM$ and $\MM\norm$ are not.
\autoref{tab:complexity-agents} on page~\pageref*{tab:complexity-agents}
summarizes our computability results for AINU, AIXI, and AINU:
iterative AINU with finite lifetime is $\Delta^0_3$.
Having an infinite horizon increases the level by one,
while restricting to $\varepsilon$-optimal policies decreases the level by one.
All versions of AINU are situated between $\Delta^0_2$ and $\Delta^0_4$
(\autoref{cor:complexity-ainu} and \autoref{cor:complexity-eps-ainu}).
For environments that almost surely continue forever
(semimeasure that are measures),
AIMU is limit-computable and $\varepsilon$-optimal AIMU is computable.
We proved that
these computability bounds on iterative AINU are generally unimprovable
(\autoref{thm:AINU_DC-is-Sigma3-hard} and \autoref{thm:AINU_LT-is-Pi2-hard}).
Additionally, we proved weaker lower bounds for AIXI
independent of the universal Turing machine
(\autoref{thm:AIXI-is-not-computable})
and for $\varepsilon$-optimal AIXI
for specific choices of the universal Turing machine
(\autoref{thm:eps-AIXI_LT-is-Sigma1-hard} and
\autoref{thm:eps-AIXI_DC-is-Pi2-hard}).

We considered $\varepsilon$-optimality in order to
avoid having to break argmax ties.
This $\varepsilon$ does not have to be constant over time,
instead we may let $\varepsilon \to 0$ as $t \to \infty$ at any computable rate.
With this we retain the computability results of $\varepsilon$-optimal policies
and get that
the value of the $\varepsilon(t)$-optimal policy $\pi^{\varepsilon(t)}_\nu$
converges rapidly to the $\nu$-optimal value:
$V^*_\nu(\ae_{<t}) - V^{\pi^{\varepsilon(t)}_\nu}_\nu(\ae_{<t}) \to 0$
as $t \to \infty$.
Therefore the limitation to $\varepsilon$-optimal policies is
not very restrictive.

When the environment $\nu$ has nonzero probability of not producing a new percept,
the iterative definition (\autoref{def:V}) of AINU fails to maximize $\nu$-expected rewards
(\autoref{prop:AINU-is-not-a-reward-maximizer}).
We introduced a recursive definition of the value function for infinite horizons
(\autoref{def:W}),
which correctly returns $\nu$-expected value.
The difference between
the iterative value function $V$ and recursive value function $W$
is readily exposed in
the difference between $M$ and $\MM$.
Just like $V$ conditions on surviving forever,
so does $\MM$ eliminate the weight of programs
that do not produce infinite strings.
Both $\MM$ and $V$ are not limit computable for this reason.

Our main motivation for the introduction of the recursive value function $W$
is the improvement of the computability of optimal policies.
Recursive AINU is $\Delta^0_3$ and
admits a limit-computable $\varepsilon$-optimal policy
(\autoref{cor:complexity-recursive-ainu}).
In this sense our goal to find a limit-computable perfect agent
has been accomplished.


\bibliographystyle{alpha}
\bibliography{../ai}

\newcommand{\etalchar}[1]{$^{#1}$}
\begin{thebibliography}{MGLA00}

\bibitem[BD62]{BD:1962}
David Blackwell and Lester Dubins.
\newblock Merging of opinions with increasing information.
\newblock {\em The Annals of Mathematical Statistics}, pages 882--886, 1962.

\bibitem[Gá83]{Gacs:1983}
Péter Gács.
\newblock On the relation between descriptional complexity and algorithmic
  probability.
\newblock {\em Theoretical Computer Science}, 22(1–2):71 -- 93, 1983.

\bibitem[Hut00]{Hutter:2000}
Marcus Hutter.
\newblock A theory of universal artificial intelligence based on algorithmic
  complexity.
\newblock Technical Report cs.AI/0004001, 2000.
\newblock \url{http://arxiv.org/abs/cs.AI/0004001}.

\bibitem[Hut01]{Hutter:2001error}
Marcus Hutter.
\newblock New error bounds for {S}olomonoff prediction.
\newblock {\em Journal of Computer and System Sciences}, 62(4):653--667, 2001.

\bibitem[Hut05]{Hutter:2005}
Marcus Hutter.
\newblock {\em Universal Artificial Intelligence: Sequential Decisions Based on
  Algorithmic Probability}.
\newblock Springer, 2005.

\bibitem[LH14]{LH:2014discounting}
Tor Lattimore and Marcus Hutter.
\newblock General time consistent discounting.
\newblock {\em Theoretical Computer Science}, 519:140--154, 2014.

\bibitem[LH15a]{LH:2015priors}
Jan Leike and Marcus Hutter.
\newblock Bad universal priors and notions of optimality.
\newblock In {\em Conference on Learning Theory}, pages 1244--1259, 2015.

\bibitem[LH15b]{LH:2015computability2}
Jan Leike and Marcus Hutter.
\newblock On the computability of {S}olomonoff induction and knowledge-seeking.
\newblock In {\em Algorithmic Learning Theory}, pages 364--378, 2015.

\bibitem[LV08]{LV:2008}
Ming Li and Paul M.~B. Vitányi.
\newblock {\em An Introduction to {K}olmogorov Complexity and Its
  Applications}.
\newblock Texts in Computer Science. Springer, 3rd edition, 2008.

\bibitem[MGLA00]{MGLA:2000}
Martin Mundhenk, Judy Goldsmith, Christopher Lusena, and Eric Allender.
\newblock Complexity of finite-horizon {M}arkov decision process problems.
\newblock {\em Journal of the ACM}, 47(4):681--720, 2000.

\bibitem[MHC99]{MHC:1999}
Omid Madani, Steve Hanks, and Anne Condon.
\newblock On the undecidability of probabilistic planning and infinite-horizon
  partially observable {M}arkov decision problems.
\newblock In {\em AAAI/IAAI}, pages 541--548, 1999.

\bibitem[MHC03]{MHC:2003}
Omid Madani, Steve Hanks, and Anne Condon.
\newblock On the undecidability of probabilistic planning and related
  stochastic optimization problems.
\newblock {\em Artificial Intelligence}, 147(1):5--34, 2003.

\bibitem[Nie09]{Nies:2009}
André Nies.
\newblock {\em Computability and Randomness}.
\newblock Oxford University Press, 2009.

\bibitem[PT87]{PT:1987}
Christos~H Papadimitriou and John~N Tsitsiklis.
\newblock The complexity of {M}arkov decision processes.
\newblock {\em Mathematics of Operations Research}, 12(3):441--450, 1987.

\bibitem[RH11]{RH:2011}
Samuel Rathmanner and Marcus Hutter.
\newblock A philosophical treatise of universal induction.
\newblock {\em Entropy}, 13(6):1076--1136, 2011.

\bibitem[SB98]{SB:1998}
Richard~S. Sutton and Andrew~G. Barto.
\newblock {\em Reinforcement Learning: An Introduction}.
\newblock MIT Press, Cambridge, MA, 1998.

\bibitem[SLR07]{SLR:2007}
Régis Sabbadin, Jérôme Lang, and Nasolo Ravoanjanahry.
\newblock Purely epistemic {M}arkov decision processes.
\newblock In {\em AAAI}, volume~22, pages 1057--1062, 2007.

\bibitem[Sol64]{Solomonoff:1964}
Ray Solomonoff.
\newblock A formal theory of inductive inference. {P}arts 1 and 2.
\newblock {\em Information and Control}, 7(1):1--22 and 224--254, 1964.

\bibitem[Sol78]{Solomonoff:1978}
Ray Solomonoff.
\newblock Complexity-based induction systems: Comparisons and convergence
  theorems.
\newblock {\em IEEE Transactions on Information Theory}, 24(4):422--432, 1978.

\bibitem[VNH{\etalchar{+}}11]{VNHUS:2011}
Joel Veness, Kee~Siong Ng, Marcus Hutter, William Uther, and David Silver.
\newblock A {M}onte-{C}arlo {AIXI} approximation.
\newblock {\em Journal of Artificial Intelligence Research}, 40(1):95--142,
  2011.

\bibitem[WSH11]{WSH:2011}
Ian Wood, Peter Sunehag, and Marcus Hutter.
\newblock ({N}on-)equivalence of universal priors.
\newblock In {\em Solomonoff 85th Memorial Conference}, pages 417--425.
  Springer, 2011.

\end{thebibliography}

\end{document}